\documentclass[letterpaper, 10pt, conference]{ieeeconf}      

\IEEEoverridecommandlockouts                              
\usepackage{amsmath}    
\usepackage{mathtools}    
\usepackage{amsfonts}
\usepackage{graphicx}   
\usepackage{subcaption}
\usepackage{epsfig} 
\usepackage{dsfont}
\usepackage{color}
\usepackage[normalem]{ulem}
\usepackage{cancel}
\usepackage{amssymb}
\usepackage{color}
\usepackage{bbm}
\usepackage{lipsum}
\let\labelindent\relax
\usepackage{enumitem}
\usepackage{siunitx}
\usepackage{placeins}
\usepackage{bbm}
\usepackage{tikz}
\usetikzlibrary{positioning}
\tikzset{main node/.style={circle,fill=blue!20,draw,minimum size=1cm,inner sep=0pt},}

\usepackage[ruled,vlined,titlenotnumbered]{algorithm2e} 
\usepackage[textsize=tiny]{todonotes}

\newcommand{\R}{\mathbb{R}} 
\newcommand{\veh}{Q} 
\newcommand{\cset}{\mathcal{U}}

\newcommand{\cfset}{\mathbb{U}}

\newcommand{\targetset}{\mathcal{L}}
\newcommand{\reachset}{\mathcal{V}}

\newcommand{\dz}{\mathcal{Z}} 
\newcommand{\Ulo}{\hat{U}} 
\newcommand{\ulo}{\hat{u}} 
\newcommand{\rcm}{c} 
\newcommand{\sv}{s} 
\newcommand{\st}{l}
\newcommand{\set}[1]{\{#1\}}            
\newcommand{\norm}[1]{\left\lVert#1\right\rVert}  

\newcommand{\target}{\mathcal{T}}
\newcommand{\vehtargetset}{\mathcal{G}}
\newcommand{\numcluster}{K}
\newcommand{\numvehicle}{N}
\newcommand{\numtarget}{M}
\newcommand{\cluster}{\mathcal{H}}

\newcommand{\rom}[1]{\uppercase\expandafter{\romannumeral #1\relax}}
\usepackage{booktabs}
\usepackage{multirow}

\usepackage{varioref}
\labelformat{algocf}{\textit{alg.}\,(#1)}

\newtheorem{defn}{Definition}

\newtheorem{thm}{Theorem}
\newtheorem{rmk}{Remark}
\newtheorem{corr}{Corollary}

\title{\LARGE \bf
Reachability-based Safe Planning for Multi-Vehicle Systems with Multiple Targets
}
\author{Jennifer C. Shih, Laurent El Ghaoui
    \thanks{Jennifer Shih and Laurent El Ghaoui are with the Department of Electrical Engineering and Computer Sciences, University of
     California, Berkeley. \{cshih, elghaoui\}@berkeley.edu}
}

\begin{document}

\maketitle
\thispagestyle{empty}
\pagestyle{empty}

\begin{abstract}
Recently there have been a lot of interests in introducing UAVs for a wide range of applications, making
ensuring safety of multi-vehicle systems a highly crucial problem. 
Hamilton-Jacobi (HJ) reachability is a promising tool for analyzing safety of vehicles for low-dimensional systems. However, reachability suffers from the curse of dimensionality, making its direct application to more than two vehicles intractable. Recent works have made it tractable to guarantee safety for 3 and 4 vehicles with reachability. However, the number of vehicles safety can be guaranteed for remains small. In this paper, we propose a novel reachability-based approach that guarantees safety for \textit{any} number of vehicles while vehicles complete their objectives of visiting multiple targets efficiently, given any $\numcluster$-vehicle collision avoidance algorithm where $\numcluster$ can in general be a small number. We achieve this by developing an approach to group vehicles into clusters efficiently and a control strategy that guarantees safety for any in-cluster and cross-cluster pair of vehicles for all time. Our proposed method is scalable to large number of vehicles with little computation overhead. We demonstrate our proposed approach with a simulation on 15 vehicles. In addition, we contribute a more general solution to the 3-vehicle collision avoidance problem from a past recent work, show that the prior work is a special case of our proposed generalization, and prove its validity. 
\end{abstract}

\section{Introduction}
In recent years, there have been vast interests from commercial companies to government agencies in introducing unmanned aerial vehicles (UAVs) into the airspace. For example, Google X \cite{Google2020}, Amazon \cite{Amazon20}, and UPS \cite{UPS2020} have all been developing drone technology for goods transport and delivery. Companies such as Zipline Inc. \cite{Zipline2020} and Vayu Inc. \cite{Vayu2020} utilize drones for delivery of critical medical supplies. The government is also tapping into UAVs for disaster response \cite{DSLRPros2020}, \cite{Humanitarion2020}, \cite{AUVSI16} and military operations \cite{Military2020}. With the burgeoning enthusiasm for this emerging technology, the Federal Aviation Administration recently devised guidelines specifically for UAVs \cite{FAA2020}. Ensuring the safety of UAVs is thus an imminent and highly impactful problem. 
A central problem in UAVs is to have them visit multiple targets for purposes such as delivery of supplies or inspection at different locations. Thus the problem of efficiently enabling all vehicles to accomplish their objectives of visiting multiple targets while maintaining safety at all times is of paramount importance.

The problem of collision avoidance among multi-agent systems has been studied through various methods. For example, \cite{Fiorini98, Vandenberg08} assume that vehicles employ specific simple control strategies to induce velocity obstacles that must be avoided by other vehicles to maintain safety. There have also been approaches that use potential functions to tackle safety while multiple agents travel along pre-determined trajectories \cite{Saber02, Chuang07}. While these approaches offer insights into tackling multi-agent problems, they do not offer the safety guarantees that are highly desirable for safety-critical systems with general dynamical systems. 

Differential game concerns the model and analysis of conflicts in dynamical systems and is a promising tool for safety-critical problems for multi-vehicle systems due to the strong theoretical guarantees it can provide. One such technique is Hamilton-Jacobi (HJ) reachability \cite{Mitchell05}. HJ reachability has been successfully used to guarantee safety for small-scale problems that concern one or two vehicles \cite{Fisac15, Mitchell05}. Despite its favorable theoretical guarantees and applicability to systems with general dynamics, it suffers from the curse of dimensionality because the computation of reachable sets grows exponentially with the number of states in the system and hence the number of vehicles, making its direct application to systems of more than two vehicles intractable. 

There have been many attempts in using differential games to analyze three-player differential games with varying-degree of assumptions on each agent in non-cooperative settings \cite{Tanimoto78, Su14, Fisac15b}. \cite{Shih16} is the first work built on reachability that guarantees safety for three vehicles while vehicles are allowed to execute \textit{any} control when the safe controller does not need to be applied, which endows vehicles more flexibility and is thus preferable in certain scenarios. \cite{Chen17} further builds on \cite{Shih16} to guarantee safety for four vehicles in unstructured settings. However, \cite{Chen17} assumes that vehicles can remove themselves from the environment when conflicts cannot be resolved for all vehicles, which is not always possible and could be undesirable in some situations. In contrast, we propose a control strategy to guarantee safety for four and more number of vehicles without assuming the ability to remove \textit{any} vehicle during conflict resolution. 

Works such as \cite{Chen15, Chen15b} have proposed controllers that guarantee safety for larger number of vehicles by imposing varying degrees of structure on the vehicles, including strong assumptions such as vehicles traveling in a single line of platoon \cite{Chen15b} or vehicles determining their trajectories \textit{a priori} \cite{Chen15}. 
In general, there is a trade-off between the number of vehicles safety can be guaranteed for and how strong the assumption on the structure of the multi-vehicle system is. In this paper, we provide a novel approach based on reachability that guarantees safety for \textit{any} number of vehicles by using less structure than those of \cite{Chen15, Chen15b} for a class of dynamical systems. Although our proposed method adopts more structure than that of \cite{Shih16} and \cite{Chen17}, our approach can guarantee safety for \textit{any} number of vehicles while avoiding having to remove vehicles from the environment when conflict cannot be resolved and retaining some level of unstructuredness. 

Our main contribution is a novel approach to guarantee safety while \textit{any} number of vehicles are tasked with visiting multiple targets for a class of dynamical systems. We first propose a method that assigns vehicles into ``teams" and induces the behavior that vehicles with similar objectives are assigned to the same team for efficiency. We then propose a control strategy to guarantee safety for any pair of vehicles within a team and across different teams, effectively guaranteeing safety for all vehicles. Furthermore, we provide a more general optimization problem that renders three-vehicle collision avoidance safe than that of \cite{Shih16} by establishing a general way of selecting the objective function in the optimization problem in \cite{Shih16} and providing a general proof for this new guideline.

\section{Background}
\label{sec:background}
In this paper, we propose to divide vehicles into teams and present a cooperative control method that guarantees safety for \textit{any} $N$  vehicles while they complete their objectives by building on \textit{any} $\numcluster$-vehicle collision avoidance algorithm where $N$ can in general be much larger than $K$. Our proposed method builds on Hamilton-Jacobi (HJ) reachability. In this section, we provide an overview of HJ reachability and the three-vehicle collision avoidance strategy proposed in \cite{Shih16}.

\subsection{Hamilton-Jacobi (HJ) Reachability \label{sec:HJI}}
HJ reachability is a promising method for ensuring safety. We give a brief overview of how HJ reachability is used to guarantee safety for a pair of vehicles as presented in \cite{Mitchell05}. For any two vehicles $\veh_i$ and $\veh_j$ with dynamics describe by the following ordinary differential equation (ODE)
\begin{equation}
\label{eq:vdyn} 
\dot{x}_m = f(x_m, u_m), \quad u_m \in \cset, m = i, j, \\
\end{equation} 
their relative dynamics can be specified by an ODE
\vspace{-0.5em}
\begin{equation}
\label{eq:rdyn} 
\begin{aligned}
\dot{\bar{x}}_{ij} &= g_{ij}(\bar{x}_{ij}, u_i, u_j), u_i,u_j \in \cset
\end{aligned}
\end{equation}
where $\bar{x}_{ij}$ is a relative state representation between $x_i$ and $x_j$ that doesn't necessarily have to be $x_i-x_j$. Note that we are using $\bar{x}_{ij}$ here because we will use $x_{ij} \equiv x_{i} - x_{j}$ throughout the paper. We assume there is a bijection between $x_{ij}$ and $\bar{x}_{ij}$.

In the reachability problem, for any pair of vehicles $\veh_i$ and $\veh_j$, we are interested in determining the backwards reachable set (BRS) $\reachset_{ij}(T)$, the set of states from which there exists no control for $\veh_i$, in the worst case non-anticipative control strategy by $\veh_j$, that can keep the system from entering some final set $\bar{\dz}_{ij}$ within a time horizon $T$. Note we will use the notation $\dz_{ij}$ such that $x_{ij} \in \dz_{ij} \Leftrightarrow \bar{x}_{ij} \in \bar{\dz}_{ij}$. For safety purpose, $\bar{\dz}_{ij}$ represents dangerous configurations between $\veh_i$ and $\veh_j$.

BRS $\reachset_{ij}$ can be mathematically described as
\begin{equation}
\label{eq:brs}
\begin{aligned}
&\reachset_{ij}(t) = \{\bar{x}_{ij}: \forall u_i \in \cfset, \exists u_j \in \cfset, \\
& \bar{x}_{ij}(\cdot) \text{ satisfies \eqref{eq:rdyn}}, \exists s \in [0, t], \bar{x}_{ij}(s) \in \bar{\dz}_{ij}\},
\end{aligned}
\end{equation}
and obtained by $\reachset_{ij}(t) = \bar{x}_{ij}: V_{ij}(t, \bar{x}_{ij}) \le 0\}$ where the details of how to obtain the value function $V_{ij}(t, \bar{x}_{ij})$ is in \cite{Mitchell05}. In this paper, we assume $t \rightarrow \infty$ and write $V_{ij}(\bar{x}_{ij}) = \lim_{t \rightarrow \infty} V_{ij}(t, \bar{x}_{ij})$. If the relative state $\bar{x}_{ij}$ of $\veh_i$ and $\veh_j$ is outside of $\reachset_{ij}$, then $\veh_i$ is safe from $\veh_j$. If $\bar{x}_{ij}$ is at the boundary of $\reachset_{ij}$, \cite{Mitchell05} shows that as long as the optimal control
\begin{equation}
    u_{ij}^* = \arg \max_{u_i \in\cset} \min_{u_j \in\cset} D_{\bar{x}_{ij}} V(\bar{x}_{ij}) \cdot g_{ij}(\bar{x}_{ij},u_i,u_j)
    \label{eq:opt_action}
\end{equation}
is applied immediately, $\veh_i$ is guaranteed to be able to avoid collision with $\veh_j$ over an infinite time horizon.

With the above in mind, we formally define the terms safety level and potential conflict:
\begin{defn}
\textbf{Safety level}: the safety level of vehicle $\veh_i$ with respect to vehicle $\veh_j$ given their relative state $\bar{x}_{ij}$ is defined as $s_{ij} \equiv V_{ij}(\bar{x}_{ij})$. 
\end{defn}
\begin{defn}
\textbf{Potential conflict}: We say vehicle $\veh_i$ is in potential conflict with vehicle $\veh_j$ if the safety level $s_{ij} \leq \st$ for some safety threshold $\st > 0$. The potential conflict is resolved when $s_{ij} > \st$.
\end{defn}

\subsection{Three-vehicle collision avoidance integer linear program (ILP) \label{sec:MIP}}
HJ reachability described in Section \ref{sec:HJI} guarantees safety for $N=2$ vehicle but applying the method directly to $N=3$ vehicles is an intractable task. \cite{Shih16} proposes an integer linear program that provides higher level control logic to guarantee safety for $N=3$ vehicles. Although our proposed method in this paper can be used with any $\numcluster$-vehicle collision avoidance algorithm that resolves potential conflicts while remaining safe, we give a brief overview of the method proposed in \cite{Shih16} because we also contribute a generalization to the approach in \cite{Shih16} in this paper.

At each time step, based on the safety value $s_{ij} \equiv V_{ij}(\bar{x}_{ij})$'s of all pairs of vehicles, \cite{Shih16} designs an integer optimization problem that solves for binary decision variables $\ulo_{ij}$ where $\ulo_{ij} = 1$ indicates that vehicle $\veh_i$ should use the optimal avoidance control $u_{ij}^*$ in Equation \eqref{eq:opt_action} to avoid $\veh_j$; if $\ulo_{ij} = 0$, vehicle $\veh_i$ does not need to avoid $\veh_j$ and can perform any action. The optimization problem has the following form
\begin{equation}
\label{eq:baseMIP}
\begin{aligned}
\max_{\ulo_{ij}} & \sum_{i, j} c_{ij} \ulo_{ij} \\
\text{subject to } & \ulo_{ij} + \ulo_{ji} \le 1 & \forall i, j, i \neq j & \quad (\ref{eq:baseMIP}a) \\
& \sum_j \ulo_{ij} \le 1 & \forall i & \quad (\ref{eq:baseMIP}b) \\
& \ulo_{ij} \in \{0, 1\} & \forall i, j, i \neq j & \quad (\ref{eq:baseMIP}c).
\end{aligned}
\end{equation}
The objective function is linear in the variables $\ulo_{ij}$'s with coefficients $c_{ij}$'s. The constraints (\ref{eq:baseMIP}a) and (\ref{eq:baseMIP}b) enforce only one vehicle in any pair of vehicles should employ avoidance control and each vehicle avoids a maximum of one other vehicle respectively. \cite{Shih16} presents a \textit{specific} numeric choice of $c_{ij}$'s to guarantee three-vehicle safety.

In this paper, we further present a general guideline for choosing the $c_{ij}$'s in the objective function of the integer program (\ref{eq:baseMIP}) and show that as long as $c_{ij}$'s satisfy the criteria we proposed, safety for three vehicles can be guaranteed. This enables a much more general and elegant proof compared to that presented in \cite{Shih16}.

\section{Problem Formulation \label{sec:formulation}}
Consider $N$ vehicles, denoted $\veh_i, i = 1, 2, \ldots, N$, with identical dynamics described by the following ordinary differential equation (ODE)
\vspace{-0.75em}
\begin{equation}
\label{eq:vdyn} 
\dot{x}_i = f(x_i, u_i), \quad u_i \in \cset, \quad i = 1,\ldots, N
\end{equation}
\noindent where $x_i \in \R^{n}$ is the state of the $i$th vehicle $\veh_i$, and $u_i$ is the control of $\veh_i$. In this paper, we work with a class of dynamical systems such that the dynamics $f$ can be described completely by a subset of the state and the control input, i.e., we can write $x_i = [x_{i,a} \text{ } x_{i,b}]$ where $x_{i,a} \in \mathbb{R}^{n_a}, x_{i,b} \in \mathbb{R}^{n_b}$, $n_a \geq 1, n_b \geq 0$, such that
\begin{equation}
    \dot{x}_i = f(x_i, u_i) = f_b(x_{i,b}, u_i), \quad u_i \in \cset, \quad i = 1,\ldots, N
    \label{eq:dynamics_custom}
\end{equation}
for some function $f_b$. Note that we will use the subscript "a" or "b" to denote the components of a given state based on the definition above throughout the paper.

Each of the $N$ vehicles is tasked with visiting a set of targets $\vehtargetset_i$, in no particular order, out of a set of $M$ targets $\set{\target_1, \dots, \target_M}$, i.e., $\vehtargetset_i \subseteq \set{\target_1, \dots, \target_M}$. Note that the exact location of each target need not to be known \textit{a priori}. Each vehicle $\veh_i$ must reach all of its targets while at all times avoid the \textit{danger zone} $\dz_{ij}$ with respect to any other vehicle $\veh_j, j = 1, \ldots, N, j \neq i$. The danger zone $\dz_{ij}$ represents relative configuration between $\veh_i$ and $\veh_j$ that are considered undesirable, such as collision. In this paper, we assume the danger zone 
$\dz_{ij}$ for each pair of vehicles can be identically defined by a \textit{norm} function on the $x_{ij,a}$ component of the relative state $x_{ij}$,  $d(x_{ij, a}): \mathbb{R}^{n_a} \rightarrow \mathcal{R}^{+}$ where $x_{ij,a} \equiv x_{i,a} - x_{j,a}$. In particular, the danger zone 
$\dz_{ij}$ is defined such that $x_{ij} \in \dz_{ij} \Leftrightarrow d(x_{ij, a}) \leq R_{ij}$ where $R_{ij}$ is some positive real number. Note that in this paper, we assume $R_{ij} = R_{ji}$ for any pair of vehicles $\veh_i, \veh_j$.
\begin{rmk}
Many practical and common dynamical systems have dynamics structures outlined in Equation (\ref{eq:dynamics_custom}), such as the 2D point system \cite{Royo19}, 3D Dubins Car \cite{Shih16}, 6D Quadrotor \cite{Royo19},  6D Acrobatic Quadrotor \cite{Gillula11}, 7D Quadrotor \cite{Royo19}, and 10D near-hover quadrotor \cite{Bouffard12}. In addition, for all these dynamical systems, defining the danger zone based on the $x_a$ component of the state $x$ makes intuitive sense as the $x_a$ components represent the x, y, z translational coordinates of these systems, which is what we generally use to define collisions among vehicles. 
\end{rmk}

Given the vehicle dynamics in \eqref{eq:vdyn}, the derived relative dynamics in \eqref{eq:rdyn}, the danger zones $\dz_{ij}, i,j = 1, \ldots, N, i \neq j$, and the sets of targets each vehicle $\veh_i$ needs to go through $\vehtargetset_i, i=1,\dots, N$, we propose a cooperative planning and control strategy that:
\begin{enumerate}
\item assigns vehicles to clusters (teams) based on their objectives;
\item determines the initial states of all vehicles;
\item guarantees safety for all vehicles for all time.
\end{enumerate}
\begin{rmk}
In this paper, we will use the terms ``cluster" and ``team" interchangeably.
\end{rmk}

Our proposed method guarantees that all vehicles will be able to stay out of the danger zone with respect to any other vehicle regardless of the number of vehicles $N$ in the environment. Additionally our method guarantees safety for all vehicles without vehicles having to remove themselves from the environment when conflicts cannot be resolved, as assumed in \cite{Chen17}. For all initial configurations, target locations, and objectives of each vehicle in our simulations, all vehicles also complete their objectives of visiting all their targets successfully. 

\section{Methodology \label{sec:method}}
Our proposed method consists of two phases: first, we develop the notion of \textit{teams (clusters)} of vehicles and present a method to assign vehicles to teams based on their targets, with the goal of minimizing the time it takes for all vehicles to complete their objectives. Second, we propose the idea of \textit{augmented} danger zone for each pair of teams. Based on this, we propose a control strategy to ensure safety for any pair vehicles on the same team and across different teams, which in combination guarantees safety for all vehicles. 

\subsection{Assignment of vehicles to clusters \label{sec:cluster_assign}}
We first propose an optimization problem that assigns the $N$ vehicles to $\numcluster$ teams, $\cluster_1, \dots, \cluster_{\numcluster}$. Each vehicle should be assigned to exactly one cluster and the objective of each cluster is then to visit, in no particular order, the union of the sets of targets of the vehicles in this cluster.
Since we aim to have our approach be applicable to scenarios where the location of each target is not known \textit{a priori}, we assume that the amount of time a cluster takes to complete its objective is proportional to the number of targets each cluster needs to visit and we don't consider the order in which each cluster visits its targets during the planning process in this paper. With this in mind, we formulate the objective function of the proposed optimization problem to
minimize the maximum number of targets each cluster needs to visit, which load-balances the number of targets each cluster should visit by grouping vehicles with similar objectives into the same cluster. Furthermore, we show that the proposed optimization problem can be converted into an integer linear program and thus solved efficiently with standard integer program solvers.  

Recall that each vehicle $\veh_i$'s objective is to visit a set of targets $\vehtargetset_i$ where $\vehtargetset_i \subseteq \set{\target_1, \dots, \target_{\numtarget}}$. Based on this, we define binary variables $e_{ij}$,  $i \in \set{1, \dots, \numvehicle}, j \in \set{1, \dots, \numtarget}$, such that $e_{ij} = \mathds{1} \{ \target_j \in \vehtargetset_i\}$ \footnote{$\mathds{1}(\mathcal{A})$ is an indicator function on event $\mathcal{A}$ such that $\mathds{1}(\mathcal{A}) = 1$ if $\mathcal{A}$ is true and $\mathds{1}(\mathcal{A}) = 0$ otherwise.}. Next we define optimization variables $y_{ik}$, $i \in \set{1, \dots, \numvehicle}, k \in \set{1, \dots, \numcluster}$, which are also binary variables. $y_{ik}=1$ means that vehicle $\veh_i$ is assigned to cluster $\cluster_k$, and $y_{ik}=0$ otherwise. Based on the goal of minimizing the maximum number of targets each cluster needs to visit as described in the previous paragraph, we propose the following optimization problem to solve for $y_{ik}$'s:
\begin{equation}
        \label{eq:orig_opt_problem}
        \begin{aligned}
          & \underset{y_{ik}}{\text{min}}
          & & \underset{k}{\text{max}} \left( \sum_{j=1}^{M} \underset{i}{max} \{ e_{ij} y_{ik} \} \right) \\
        & \text{subject to}
        & & \sum_{k=1}^{\numcluster} y_{ik} = 1, \forall i \in \set{1, \dots, \numvehicle} \\
        & & &  y_{ik} \in \set{0, 1}, \forall i \in \set{1, \dots, \numvehicle}, \forall k \in \set{1, \dots, \numcluster}
        \end{aligned}
\end{equation}
Note that due to space constraints under the $min, max$ notations in the objective, we omit that we're optimizing over $y_{ik}, \forall i \in \set{1, \dots, \numvehicle}, \forall k \in \set{1, \dots, \numcluster}$ for the minimization and $k, \forall k \in \set{1, \dots, \numcluster}$ and $i, \forall i \in \set{1, \dots, \numvehicle}$ for the maximization in the above optimization problem.

The summation $\sum_{j=1}^{M} \underset{i}{max} \{ e_{ij} y_{ik} \}$ is equivalent to the total number of targets that cluster $\cluster_k$ needs to visit. To see this, for a given target $\target_j$, the term $ e_{ij} y_{ik}$ in the summation equals to $1$ if vehicle $\veh_i$ needs to visit target $\target_j$ \textit{and} $\veh_i$ is assigned to cluster $\cluster_k$. $e_{ij} y_{ik} = 0$ otherwise. Hence $\underset{i}{max} \{ e_{ij} y_{ik} \}$ equals to $1$ if at least one vehicle assigned to cluster $\cluster_k$ needs to visit target $\target_j$. $\underset{i}{max} \{ e_{ij} y_{ik} \} = 0$ otherwise. Summing $\underset{i}{max} \{ e_{ij} y_{ik} \}$ over all targets gives the total number of targets cluster $\cluster_k$ needs to visit. 

Next we show that the optimization problem \eqref{eq:orig_opt_problem} can be converted into a standard integer linear program by introducing a slack variable and an inequality constraint for each of the maximization operations in the objective.
\begin{equation*}
        \label{eq:converted_opt_problem}
        \begin{aligned}
          & \underset{y_{ik}, o_{kj}, O}{\text{min}}
          & & O \\
        & \text{subject to}
        & & \sum_{k=1}^{\numcluster} y_{ik} = 1, \forall i \in \set{1, \dots, \numvehicle} \\
        & & &  y_{ik} \in \set{0, 1}, \forall i \in \set{1, \dots, \numvehicle}, \forall k \in \set{1, \dots, \numcluster} \\
        & & & e_{ij} y_{ik} \leq o_{kj}, \forall i \in \set{1, \dots, \numvehicle}, \forall k \in \set{1, \dots, \numcluster},  \\ 
        & & & \forall j \in \set{1, \dots, \numtarget} \\
        & & & \sum_{j=1}^{M} o_{kj} \leq O, \forall k \in \set{1, \dots, \numcluster}.
        \end{aligned}
\end{equation*}
The above integer linear problem can be solved efficiently by off-the-shelf integer program solvers. Once solved, the values of $y_{ik}$'s are the solution to the team assignment problem. This completes the first step of the planning process. 

\subsection{Collision Avoidance Protocol Design}
In this section, we present our proposed control strategy that ensure all vehicles remain safe when completing their objectives after the vehicles have been assigned to teams. Specifically, given any $\numcluster$-vehicle collision avoidance algorithm that guarantees safety when resolving potential conflicts among $\numcluster$ vehicles, we propose a general way to initialize vehicle locations and a safe control strategy such that the following always hold for $\numvehicle$ vehicles where $\numvehicle$ can be much larger than $\numcluster$:
\begin{itemize}
    \item Any vehicle is safe from any other vehicle within the same cluster.
    \item Any vehicle in a cluster is safe from any vehicle in any other cluster.
\end{itemize}

\subsubsection{Guaranteed safety for all vehicles within the same cluster}
We first prove a theorem that motivates the control strategy that enables any pair of vehicles in the same cluster to remain safe from each other. 
\begin{thm}
    \label{thm:dist_same}
    Give the structure of the dynamics and the danger zone defined in Section \ref{sec:formulation}, for any two vehicles $\veh_i$ and $\veh_j$, if the initial states $x_i(t_0), x_j(t_0)$ of the two vehicles satisfy $d(x_{ij,a}(t_0)) > R_{ij}$ and $x_{i,b}(t_0) = x_{j,b}(t_0)$ and the controls of the vehicles satisfy $u_i(t) = u_j(t)\text{ } \forall t \geq t_0$, then vehicles $\veh_i$ and $\veh_j$ will remain safe from each other for all $t > t_0$.
\end{thm}
\begin{proof}
    Given that $x_{i,b}(t_0) = x_{j,b}(t_0)$ and $u_i(t) = u_j(t) \text{ } \forall t \geq t_0$, we have that at any time $t \geq t_0$, $\dot{x}_i(t) = f(x_i(t), u_i(t)) = f_b(x_{i,b}(t), u_i(t)) = f_b(x_{j,b}(t), u_j(t)) = f(x_j(t), u_j(t)) = \dot{x}_j(t)$. Because $\dot{x}_{ij,a}(t) = 0 \text{ } \forall t \geq t_0$, $x_{ij,a}(t) = x_{ij,b}(t_0) \text{ } \forall t > t_0$. Thus, $d(x_{ij,a}(t)) = d(x_{ij,a}(t_0)) > R_{ij} \text{ } \forall t > t_0$. Since $R_{ij} = R_{ji}$ and $x_{ij}(t) = -x_{ji}(t)$, we have $d(x_{ij,a}(t)) = d(x_{ji,a}(t)) > R_{ij} = R_{ji} \text{ } \forall t \geq t_0$, which proves that $\veh_i$ and $\veh_j$ will remain safe from each other for all $t > t_0$.
\end{proof}
The above shows that if we initialize any pair of vehicles $\veh_i, \veh_j$ in the same cluster such that $x_{i,b}(t_0) = x_{j,b}(t_0)$, vehicles $\veh_i, \veh_j$ start out safe from each other, and that they employ the same control at any time, the two vehicles will continue to remain outside of each other's danger zone for all time. We can directly use this insight to initialize all vehicles in the same cluster such that any pair of vehicles in the same cluster satisfies the above conditions and have all vehicles in the same cluster employ the same control to guarantee safety for all vehicles in the same cluster for all time. 

\subsubsection{Guaranteed safety of any vehicle with respect to any other vehicle in a \textit{different} cluster}
The key idea of our proposed method is that we can think of each cluster $\cluster_k$ as an \textit{imaginary} vehicle with state $x_{\cluster_k}$ and dynamics identical to that of the individual vehicle's dynamics. We propose the concept of \textit{augmented} danger zone between any pair of clusters, which allows us to guarantee that any vehicle in a cluster will remain safe from any vehicle in any other cluster. 

Before we proceed to describe our approach, we first define a few essential terms:
\begin{defn}
\textbf{Maximum vehicle distance to cluster center} for cluster $\cluster_k$ is defined as $R_{\cluster_k} \equiv \underset{i: \veh_i \in \cluster_k}{max} \text{ } d(x_{\cluster_k,a} - x_{i,a})$ where $x_{\cluster_k}$ is the state of the \textit{imaginary} vehicle representing cluster $\cluster_k$.
\end{defn}

\begin{defn}
\textbf{Augmented danger zone} $\mathcal{Z}_{\cluster_k \cluster_l}$ of cluster $\cluster_k$ with respect to  $\cluster_l$ is defined as $x_{\cluster_k \cluster_l} \in \mathcal{Z}_{\cluster_k \cluster_l}  \Leftrightarrow d(x_{\cluster_k \cluster_l, a}) \leq R_{\cluster_k \cluster_l} $ where $x_{\cluster_k \cluster_l} = x_{\cluster_k} - x_{\cluster_l}$ and $R_{\cluster_k \cluster_l} = R_{\cluster_k} + R_{\cluster_l} + \underset{\veh_i \in \cluster_k, \veh_j \in \cluster_l}{max} \text{ } R_{ij}$. And note that $x_{\cluster_k \cluster_l} \in \mathcal{Z}_{\cluster_k \cluster_l} \Leftrightarrow \bar{x}_{\cluster_k \cluster_l}  \in \bar{\mathcal{Z}}_{\cluster_k \cluster_l}$.
\end{defn}

\begin{defn}
\textbf{Safety level} of cluster $\cluster_k$ with respect to $\cluster_l$ is defined as $s_{\cluster_k \cluster_l} \equiv V_{\cluster_k \cluster_l}(\bar{x}_{\cluster_k \cluster_l})$ where $V_{\cluster_k \cluster_l}(\bar{x}_{\cluster_k \cluster_l})$ is computed based on reachability computation described in Section \ref{sec:HJI} with dynamics identical to that of the vehicle dynamics and danger zone $\bar{\mathcal{Z}}_{\cluster_k \cluster_l}$.
\end{defn}

Now we prove a result that relates the danger zone of the \textit{imaginary} vehicles representing the clusters and the danger zone of the actual vehicles. 
\begin{thm}
If $x_{\cluster_k \cluster_l} \notin \mathcal{Z}_{\cluster_k \cluster_l}$, then $x_{ij} \notin \mathcal{Z}_{ij}$ for any pair of vehicles $\veh_i, \veh_j$ such that $\veh_i \in \cluster_k$ and $\veh_j \in \cluster_l$.
\label{thm:safe_other_cluster}
\end{thm}
\begin{proof}
Let $r_{\cluster_k} = x_{\cluster_k} - x_{i}$ and $r_{\cluster_l} = x_{\cluster_l} - x_{j}$. Based on the definition of the augmented danger zone $\mathcal{Z}_{\cluster_k \cluster_l}$, we have $x_{\cluster_k \cluster_l} \notin \mathcal{Z}_{\cluster_k, \cluster_l}  \Leftrightarrow d(x_{\cluster_k \cluster_l, a}) > R_{\cluster_k \cluster_l}$. With this in mind, we have 
\begin{align*}
    d(x_{\cluster_k \cluster_l, a})
&= d(x_{\cluster_k,a} - x_{\cluster_l, a}) \\
&= d(x_{i,a} + r_{\cluster_k,a} - x_{j,a} - r_{\cluster_l,a}) \\ 
&\leq d(x_{i,a} - x_{j,a}) + d(r_{\cluster_k,a}) + d(r_{\cluster_l,a})  \\
&\leq d(x_{ij,a}) + R_{\cluster_k} + R_{\cluster_l}
\end{align*}
where the first inequality follows from the triangle inequality on norms and the second inequality follows from the definitions of $R_{\cluster_k}$ and $R_{\cluster_l}$. 
Hence we have
\begin{align*}
     d(x_{ij,a}) + R_{\cluster_k} + R_{\cluster_l} 
     &\geq  d(x_{\cluster_k \cluster_l, a}) \\
     & > R_{\cluster_k \cluster_l} \\
     & = R_{\cluster_k} +R_{\cluster_l}+ \underset{\veh_i \in \cluster_k, \veh_j \in \cluster_l}{max} \text{ } R_{ij} \\
     &\geq R_{\cluster_k} +R_{\cluster_l}+ R_{ij}.
\end{align*}
Subtracting $R_{\cluster_k} +R_{\cluster_l}$ from both sides results in $d(x_{ij,a}) >  R_{ij}$, which implies $x_{ij} \notin \mathcal{Z}_{ij}$, as desired.
\end{proof}

\begin{corr}
Suppose at time $t=t_0$, for any cluster $\cluster_k$, $x_{\cluster_k, b}(t_0) = x_{i,b}(t_0)$ for all $i$ such that $\veh_i \in \cluster_k$. We apply the $\numcluster$-vehicle collision avoidance strategy that guarantees safety on the $\numcluster$ imaginary vehicles representing the $\numcluster$ clusters when resolving potential conflicts. If the strategy suggests to apply $u^{\star}_{\cluster_k}$ to the imaginary vehicle representing cluster $\cluster_k$, then in additional to applying this control on the imaginary vehicle, we also apply this control to all vehicles in this cluster. Given the aforementioned assumptions and the control strategy, if at time $t=t_0$, any pair of  imaginary vehicles representing two distinct clusters $\cluster_k, \cluster_l$ are not in potential conflict with each other, for any pair of vehicles $\veh_i \in \cluster_k, \veh_j \in \cluster_l$, $\veh_i$ will remain safe from $\veh_j$ for all time $t \geq t_0$.
\label{corr:safe_other_cluster}
\end{corr}
\begin{proof}
First we note that it is only possible to have the same or less number of vehicles in a cluster as time proceeds because a vehicle is allowed to stay at its final target once it completes visiting all its targets.
In addition, the same control is applied to all vehicles in any cluster $\cluster_k$ and the imaginary vehicle representing $\cluster_k$.
Thus the maximum vehicle distance to cluster center $R_{\cluster_k}$ for each cluster $\cluster_k$ is non-increasing throughout execution, which means that the radius $R_{\cluster_k \cluster_l}$ defining the \textit{augmented} danger zone between any two distinct clusters $\cluster_k, \cluster_l$ is non-increasing. By applying the $\numcluster$-vehicle collision avoidance control strategy on the \textit{imaginary} vehicles representing the $\numcluster$ clusters, we know that for any distinct clusters $\cluster_k, \cluster_l$, we have $x_{\cluster_k \cluster_l} \notin \mathcal{Z}_{\cluster_k \cluster_l}$ for all $t \geq t_0$ under the assumption that they are not in potential conflict initially. Applying Theorem \ref{thm:safe_other_cluster}, we have that $x_{ij} \notin \mathcal{Z}_{ij}$ for any vehicle $\veh_i \in \cluster_k, \veh_j \in \cluster_l$, which implies that any vehicle with respect to any vehicle in another cluster will remain safe from each other for all $t \geq t_0$.
\end{proof}

With the above in mind, we summarize our proposed overall initialization and cooperative control strategy for all vehicles to visit all their targets safely for all time:
\begin{itemize}
    \item (1) Initialize all vehicles such that for any pair of vehicles $\veh_i, \veh_j$ in the same cluster $\cluster_k$, $x_{ij} \notin \mathcal{Z}_{ij}$ and $x_{\cluster_k,b}(t_0)=x_{i,b}(t_0) = x_{j,b}(t_0)$. Additionally, any two distinct clusters $\cluster_k$ and $\cluster_l$ are initialized so that the imaginary vehicles representing them are not in potential conflict with each other.
    \item (2) At any time $t$, for any cluster $\cluster_k$, if the $\numcluster$-vehicle collision avoidance algorithm determines it's necessary to apply the optimal safety controller, then all vehicles in $\cluster_k$ apply this safe control; if the $\numcluster$-vehicle collision avoidance algorithm determines that no safety control is needed at this time step, all vehicles in $\cluster_k$ apply the target controller that gets the cluster to its next target. 
\end{itemize}

The target controller is obtained by first computing the optimal control, up to discretization accuracy, to reach the goal for any relative state of a vehicle and the goal within a finite grid using reachability offline. Online, all is needed to get the current target control is to look up the optimal control using the current relative state of the cluster and its next goal location. Hence the target locations need not to be known \textit{a priori}.

\begin{corr}
Give the control strategy outlined above, all vehicles will remain safe from each other for all time.
\end{corr}
\begin{proof}
The above initialization and control strategy satisfy the assumptions of both Theorem \ref{thm:dist_same} and Corollary \ref{corr:safe_other_cluster}. Since the union of any pair of vehicles within the same cluster and across different clusters is exactly all pairs of vehicles, any pair of vehicles will remain safe from each other for all time.
\end{proof}

\subsection{Improvement of generality of proposed optimization problem for 3-vehicle collision avoidance in \cite{Shih16}}

In this section, we provide a more general optimization problem for guaranteeing safety for three vehicles compared to that presented in \cite{Shih16}, which only presents a specific numeric choice of the $c_{ij}$'s in the objective function of ILP (\ref{eq:baseMIP}) that happens to work by verifying through enumeration and brute-force, offering no general guideline on what makes choices of $c_{ij}$'s sufficient for safety guarantees. In this paper, we develop sufficient conditions for $c_{ij}$'s for safety to be guaranteed for $N=3$ vehicles. 

First, we set $c_{ii} = -1 \text{ } \forall i$ and $c_{ij} = -1$ when $\sv_{ij} > \st, i \neq j$ where $\st$ is the safety threshold defined in Section \ref{sec:background}. Next we illustrate our proposed design of $c_{ij}$ when $s_{ij} \leq \st$ with the following theorem.

\begin{thm}
\label{thm:main_result}
Let $N = 3$ and suppose $\sv_{12}, \sv_{23}, \sv_{31} > 0$ at some time $t = t_0$. If the joint control strategy from the integer program \eqref{eq:baseMIP} has reward coefficient elements $\rcm_{ij}$'s that satisfy the conditions, 

\vspace{-0.8em}
\begin{align} \label{eq:rcm_3veh}
    \sv_{12} \leq \st \Rightarrow \rcm_{12} &> \rcm_{13}^+  + \rcm_{21}^+ +
    \rcm_{32}^+ \\
    \sv_{23} \leq \st \Rightarrow  \rcm_{23} &> \rcm_{13}^+ + \rcm_{21}^+ +
    \rcm_{32}^+ \\
    \sv_{31} \leq \st \Rightarrow   \rcm_{31} &> \rcm_{13}^+ + \rcm_{21}^+ +
    \rcm_{32}^+ 
\end{align}
where $\rcm_{ij}^+ = max(\rcm_{ij}, 0)$,
it is guaranteed that $\sv_{12}, \sv_{23}, \sv_{31} > 0 ~ \forall t > t_0$.
\end{thm}
\begin{proof}
  Observe that we can use a graph to represent the constraints in ILP \eqref{eq:baseMIP}.
  
   \begin{center}
     \begin{tikzpicture}

       \node[main node, fill=white, minimum size=2pt] (1) {$\ulo_{12}$};
       \node[main node, fill=white, minimum size=2pt] (2) [below left = 0.635cm and 0.415cm of 1]  {$\ulo_{13}$};
       \node[main node, fill=white, minimum size=2pt] (3) [right = 0.835cm of 1] {$\ulo_{21}$};
       \node[main node, fill=white, minimum size=2pt] (4) [below right = 0.635cm and 0.415cm of 3] {$\ulo_{23}$};
       \node[main node, fill=white, minimum size=2pt] (5) [below left = 0.635cm and 0.415cm of 4] {$\ulo_{32}$};
       \node[main node, fill=white, minimum size=2pt] (6) [left = 0.835cm of 5] {$\ulo_{31}$};

       \path[draw,thick]
       (1) edge node { } (2)
       (2) edge node { } (6)
       (3) edge node { } (4)
       (4) edge node { } (5)
       (5) edge node { } (6)
       (3) edge node { } (1)
       ;
     \end{tikzpicture}
   \end{center}

    Each vertex represents a variable $\ulo_{ij}, i \neq j$ in the optimization problem. 
    An edge between node $\ulo_{ij}$ and $\ulo_{kq}$ exists if and only if the constraint $\ulo_{ij} + \ulo_{kq} \leq 1$ is in the linear constraints in ILP \eqref{eq:baseMIP} when we eliminate considering $\ulo_{ii}$ as its optimal value is $0$ for all $i$ trivially.
    
    It suffices to show that $0 < \sv_{12}, \sv_{23}, \sv_{31} \le \st$ at $t = t_0$ implies $\sv_{12}, \sv_{23}, \sv_{31} > 0 ~ \forall t > t_0$.
    Let $\Ulo^*$ denote the optimal solution to ILP \eqref{eq:baseMIP}, and assume $\sv_{12}, \sv_{23}, \sv_{31} >  0$ at time $t = t_0$. Based on Proposition 1 in \cite{Shih16}, our goal is to show
    \begin{equation*}
      \sv_{12} \leq \st \Rightarrow \ulo^*_{12} = 1, 
      \sv_{23} \leq \st \Rightarrow \ulo^*_{23} = 1, 
      \sv_{31} \leq \st \Rightarrow \ulo^*_{31} = 1.
    \end{equation*}    
    
    Without loss of generality (WLOG), we prove that $\ulo^*_{12}=1$ whenever $\sv_{12} \leq \st$. 
    Consider the following three cases when $\sv_{12} \leq \st$:
    \begin{itemize}
      \item $\sv_{23}, \sv_{31} \leq \st$: 
        By \eqref{eq:rcm_3veh}, we have $\rcm_{12},\rcm_{23},\rcm_{31} > \rcm_{13}^+ + \rcm_{21}^+ +
        \rcm_{32}^+$. From the constraint graph, one can see that the maximum number of non-neighboring variables that can take on values of $1$ is three. Thus it's clear that $\ulo^*_{12}=\ulo^*_{23}=\ulo^*_{31}=1, \ulo^*_{13}=\ulo^*_{21} = \ulo^*_{32}=0$ yields the largest possible objective while being feasible.
      \item Exactly one of the inequalities $\sv_{23} \leq \st$, $\sv_{31} \leq \st$ is true: Assume WLOG that $\sv_{23}
        \leq \st$ and $\sv_{31} > \st$. This gives us $\ulo_{31}^{*} = 0$. With $\ulo_{31}^{*} = 0$, regardless of the values of $\rcm_{13}, \rcm_{21}, \rcm_{32}$, we always have $\ulo_{12}^*=\ulo_{23}^*=1$. This is because first, $\rcm_{12} + \rcm_{23}$ is always greater than $\rcm_{12}$ or $\rcm_{23}$ alone as they are both positive. Second, $\rcm_{12} + \rcm_{23}$ is also always greater than the sum of any feasible combination of $\rcm_{13},\rcm_{21},\rcm_{32}$ or the sum of exactly one of $\rcm_{12}, \rcm_{23}$ plus any feasible combination of $\rcm_{13},\rcm_{21},\rcm_{32}$ due to the condition that $\rcm_{12}, \rcm_{23} > \rcm_{13}^+ + \rcm_{21}^+ +\rcm_{32}^+$.
      \item $\sv_{23},\sv_{31} > \st$: We have $\ulo^*_{23} = \ulo^*_{31} = 0$. 
        Based on \eqref{eq:rcm_3veh}, it's clear that the optimizer always has $\ulo^*_{12} = 1$ because regardless of the values of $\rcm_{13},\rcm_{21}, \rcm_{32}$, $\rcm_{13}^+ + \rcm_{21}^+ + \rcm_{32}^+$ is always less than $\rcm_{12}$.
    \end{itemize}

  In summary, when $\sv_{12} \leq \st$, we always have $\ulo^*_{12} = 1$. By a similar argument, $\sv_{23} \leq \st \Rightarrow \ulo^*_{23} = 1$ and $\sv_{31} \leq \st \Rightarrow \ulo^*_{31} = 1$ hold.

\end{proof}

 \begin{rmk}
    One can notice that the specific numeric choice of $\rcm_{ij}$'s presented in \cite{Shih16} satisfies the newly proposed general guideline in this paper and is thus a special case of the guideline.
  \end{rmk}

\section{Numerical Simulations}
\label{sec:simulation}
We demonstrate our proposed approach on safe planning and control for multiple vehicles, each with an objective of visiting multiple targets, in simulation. We show that our approach enables guaranteed safety for $N=4$ vehicles without the need to remove any vehicle in the environment like it is assumed in \cite{Chen17}. In addition, we also demonstrate that our approach scales easily to large number of vehicles by demonstrating it on $N=15$ vehicles. In all our simulations, we divide vehicles into $\numcluster=3$ clusters and build on the $3$-vehicle collision avoidance algorithm in \cite{Shih16}. 

For illustration purposes, we assumed that the dynamics of each vehicle $\veh_i$ is given by
\vspace{-0.5em}
\begin{equation}
\dot{p}_{x,i} = v \cos \theta_i, \text{ } \dot{p}_{y,i} = v \sin \theta_i, \text{ } \dot{\theta}_i = \omega_i, \quad |\omega_i| \le \bar{\omega}
\vspace{-0.5em}
\end{equation}
\noindent where the state variables $p_{x,i}, p_{y,i}, \theta_i$ represent the $x$ position, $y$ position, and heading of vehicle $\veh_i$. Each vehicle travels at a constant speed of $v=5$, and chooses its turn rate $\omega_i$, constrained by maximum $\bar{\omega}=1$. The danger zone for HJ computation between $\veh_i$ and $\veh_j$ is defined as
\begin{equation}
\targetset_{ij} = \{x_{ij}: (p_{x,i} - p_{x,j})^2 + (p_{y,i} - p_{y,j})^2 \leq R_c^2\},
\end{equation}
\noindent whose interpretation is that $\veh_i$ and $\veh_j$ are considered to be in each other's danger zone if their positions are within $R_c$ of each other. Here, $x_{ij} = [p_{x,ij}, p_{y,ij}, \theta_{ij}]=[p_{x,i} - p_{x,j}, p_{y,i} - p_{y,j}, \theta_i - \theta_j]$. The danger zone can be equivalently defined by the L-2 norm of the $x$ and $y$ components of the states, i.e., $x_{ij} \in \mathcal{Z}_{ij}$ if and only if $d(x_{ij,a}) = \norm{x_{ij,a}}_2 \leq R_{c}$.

To obtain safety levels and the optimal pairwise safety controller, we compute the BRS \eqref{eq:brs} with the relative dynamics
\vspace{-0.5em}
\begin{equation}
  \label{eq:dyn_ij}
  \begin{aligned}
  \dot{q}_{x, ij} &=  -v + v \cos q_{\theta,ij} + \omega_i q_{y, ij} \\
  \dot{q}_{y, ij} &= v \sin q_{\theta,ij} - \omega_i q_{x, ij} \\
  \dot{q}_{\theta,ij} &= \omega_j - \omega_i, \quad |\omega_i|, |\omega_j| \le \bar{\omega}
  \end{aligned}
  \vspace{-0.5em}
\end{equation}
where $[q_{x,ij},q_{y,ij}]$ is $[-p_{x,ij}, -p_{y,ij}]$ rotated clockwise by $\theta_{i}$ around the origin on the 2D plane and $q_{\theta,ij}=-\theta_{ij}$. Note that the L-2 norm on $[q_{x,ij}, \text{ } q_{y,ij}]$ is the same as the L-2 norm on $[p_{x,ij}, p_{y,ij}]$ because changing the sign and rotating do not change the value of the norm so we could have similarly defined the danger zone as $\bar{\mathcal{Z}}_{ij} = \set{ \bar{x}_{ij}: \norm{[q_{x,ij}, \text{ } q_{y,ij}]}_2 \leq R_c }$ where $\bar{x}_{ij}=[q_{x,ij}, q_{y,ij}, q_{\theta,ij}]$.

For all simulation, we initialize all vehicles and states of the clusters such that any pair of vehicles in the same cluster is of distance greater than $R_c$ of each other and the pairwise safety levels of any two distinct clusters based on the \textit{augmented danger zones} between them are all above the safety threshold $\st = 1.5$. 

In Figure \ref{fig:four_veh}, we provide snapshots of the simulation of our proposed approach on $4$ vehicles in an environment with $4$ targets. In this simulation, the set of targets each vehicle needs to visit is $\veh_1: [A, D]$, $\veh_2: [B]$, $\veh_3: [C]$, $\veh_4: [D]$. By using our proposed team assignment algorithm presented in Section \ref{sec:cluster_assign}, the three clusters $\cluster_1$, $\cluster_2$, $\cluster_3$ have the following vehicles assigned to them, $\cluster_1: \veh_1, \veh_4$, $\cluster_2: \veh_2$, $\cluster_3: \veh_3$. Recall that the set of targets for each cluster is the union of the targets of all vehicles in the cluster. Hence $\cluster_1$ should visit targets $[A, D]$, $\cluster_2$ should visit target $[B]$, and $\cluster_3$ should visit target $[C]$. We see that the team assignment algorithm offers a solution such that no clusters have to visit more than $2$ targets to encourage efficient completion of the objectives of all vehicles. If $\veh_1$ was paired with either $\veh_2$ or $\veh_3$ instead, one cluster would have to visit $3$ targets.

In this simulation, the danger zone radius is $R_c = 3$. For cluster $\cluster_1$, we choose the state $x_{1}$ of vehicle $\veh_1$ to be identical to $x_{\cluster_1}$, the state of the \textit{imaginary} vehicle representing the cluster, and choose $\veh_4$ to be at a distance of $R_c + \epsilon$ from the cluster center where $\epsilon$ is a small positive real number. Hence $R_{\cluster_1} = \underset{i \in \set{1, 4}}{max} \text{ } d(x_{\cluster_1}- x_i) = 3 + \epsilon$. For clusters $\cluster_2$ and $\cluster_3$, the state of the \textit{imaginary} vehicle is the state of the only vehicle in each cluster, i.e., $x_{\cluster_2} = x_2, x_{\cluster_3} = x_3$. Hence $R_{\cluster_2} = R_{\cluster_3} = 0$. For each cluster $\cluster_k$, a circle with radius $R_{\cluster_k}$ centered at $x_{\cluster_k}$ is plotted if $R_{\cluster_k} > 0$. We also plot the $0$-safety level reachable sets derived from the \textit{augmented danger zones} of the clusters around the cluster centers. We can see from the top two subplots in Figure \ref{fig:four_veh} that the $0$-safety level sets corresponding to $V_{\cluster_1 \cluster_2}$ and $V_{\cluster_3 \cluster_1}$ are greater than that of $V_{\cluster_2 \cluster_3}$ because the radii $R_{\cluster_{1} \cluster_{2}} , R_{\cluster_3 \cluster_1}$ that define their \textit{augmented danger zones} are $R_{\cluster_{1} \cluster_{2}} = R_{\cluster_3 \cluster_1} = R_{\cluster_1}+R_{\cluster_2}+R_c=R_{\cluster_3}+R_{\cluster_1}+R_c = 6 + \epsilon$ while the radius $R_{\cluster_2 \cluster_3}$ defining the \textit{augmented danger zone} between $\cluster_2$ and $\cluster_3$ is $R_{\cluster_2 \cluster_3} = R_{\cluster_2}+R_{\cluster_3}+R_c = 3$. 

In Figure \ref{fig:four_veh}, as the clusters move towards their first targets, they get into potential conflicts with each other. Hence the safety control kicks in. After each cluster successfully resolves the conflict, $\cluster_2$ heads to target $B$, $\cluster_3$ heads to $C$, and $\cluster_1$ first goes to target $D$, followed by target $A$. At time $t=14.5$s, we see that all vehicles have completed their objectives. Note that once a vehicle has visited all its targets, it remains at its last visited target and is no longer considered for collision avoidance. 

\begin{table}[h]
\centering
\resizebox{\columnwidth}{!}{
\begin{tabular}{ |c|c|c|c|  }
 \hline
 Vehicle & Vehicle Targets & Cluster & Cluster Targets \\
 \hline
    $\veh_6$   & [F, G, H] &  &  \\
    $\veh_7$ &  [H, I]  &   & \\
    $\veh_8$ &  [H, I, J]  &  $\cluster_1 (red)$ &  [F, G, H, I, J, M] \\
    $\veh_{10}$ &  [I, M]  &   & \\
    $\veh_{14}$ &  [J]  &   & \\
    \hline
    $\veh_1$ &  [A, C, E]  &   & \\
    $\veh_2$ &  [A, C]  &   & \\
    $\veh_4$ &  [B, C, D]  &  $\cluster_2 (green)$  & [A, B, C, D, E, G] \\
    $\veh_5$ &  [B, E]  &   & \\
    $\veh_{9}$ &  [B, D, G]  &   & \\
    $\veh_{15}$ &  [C, E]  &   & \\
    \hline
    $\veh_{3}$ & [P, K, O]   &   & \\
    $\veh_{11}$ & [P]   &   $\cluster_3 (blue)$ & [P, A, F, K, O, N] \\
    $\veh_{12}$ & [A, F]   &   & \\
    $\veh_{13}$ & [O, N]   &   & \\
 \hline
\end{tabular}}
\label{tab:fifteen_veh_details}
\caption{This table summarizes the targets for each vehicle, the cluster each vehicle is assigned to based on the proposed cluster assignment algorithm, and the targets that each cluster should visit for the $15$-vehicle collision avoidance problem. We see that the solution to the cluster assignment successfully minimized the maximum number of targets each cluster needs to visit and load balances it so that each cluster needs to visit the same number of targets.}
\end{table}

We demonstrate the scalability and effectiveness of our proposed method with a simulation on getting 15 vehicles to complete their objectives where there are 16 targets in the environment. In this simulation, the danger zone radius is $R_c = 2$. The targets of each vehicle and the cluster assignments from running our proposed team assignment algorithm are summarized in Table 1. We see that our proposed assignment algorithm successfully divides the vehicles into three clusters such that the number of targets each cluster needs to visit is well-balanced. Each cluster visits the targets in the order under the column ``Cluster Targets" in Table 1. The top left graph in Figure \ref{fig:fifteen_veh} shows the starting configuration of the vehicles where the initialization scheme is similar to that explained for the four-vehicle simulation: cluster $\cluster_1$ (red) has its center at $\veh_{10}$, i.e., $x_{\cluster_1} = x_{\veh_{10}}$ and the rest of the vehicles in the cluster are located at equal distance to each other on a circle of radius $R_c=2+\epsilon$ centered at the cluster center. Similarly, for cluster $\cluster_2$ (green) and $\cluster_3$ (blue), the cluster center is located at where vehicles $\veh_9$ and $\veh_{13}$ are at respectively, and the rest of the vehicles in each cluster are located at equal distance to each other on a circle of radius $R_c=2+\epsilon$. In general, we make the state of the imaginary vehicle representing the cluster identical to the state of the vehicle that completes its objective last in the cluster. We see that our proposed method resolves all conflicts and all $15$ vehicles complete their objectives of visiting their targets while maintaining safety successfully. 

For the 15-vehicle simulation, it takes on average 0.018 seconds to perform computation at each time step. All computations were done on a MacBookPro 15.1 laptop with an Intel Core i7 processor.

\begin{figure}[]
\centering
  \begin{subfigure}[b]{0.23\textwidth}
    \includegraphics[width=\textwidth]{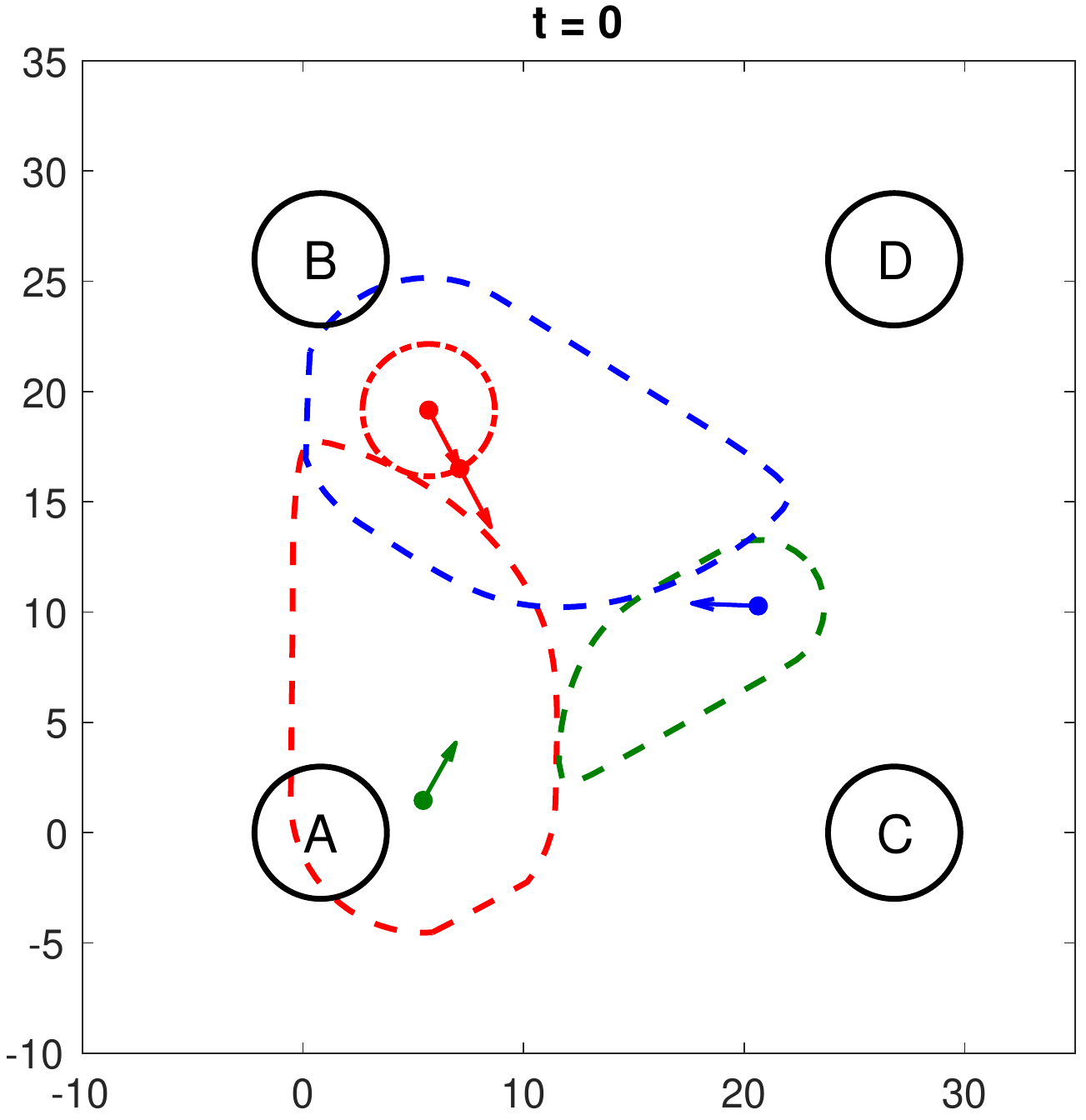}
  \end{subfigure}
  \begin{subfigure}[b]{0.23\textwidth}
    \includegraphics[width=\textwidth]{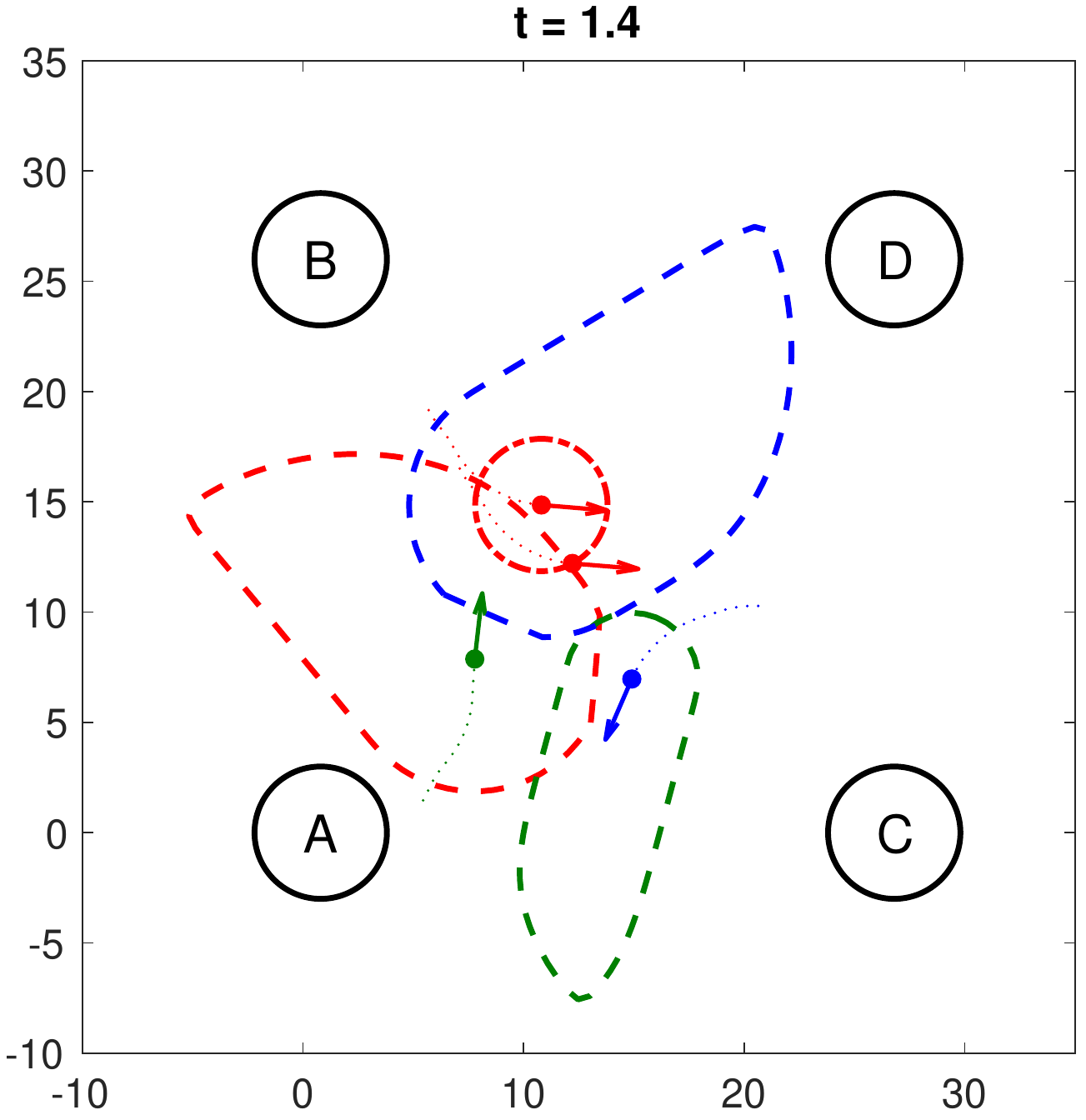}
  \end{subfigure}
  \\
  \begin{subfigure}[b]{0.23\textwidth}
    \includegraphics[width=\textwidth]{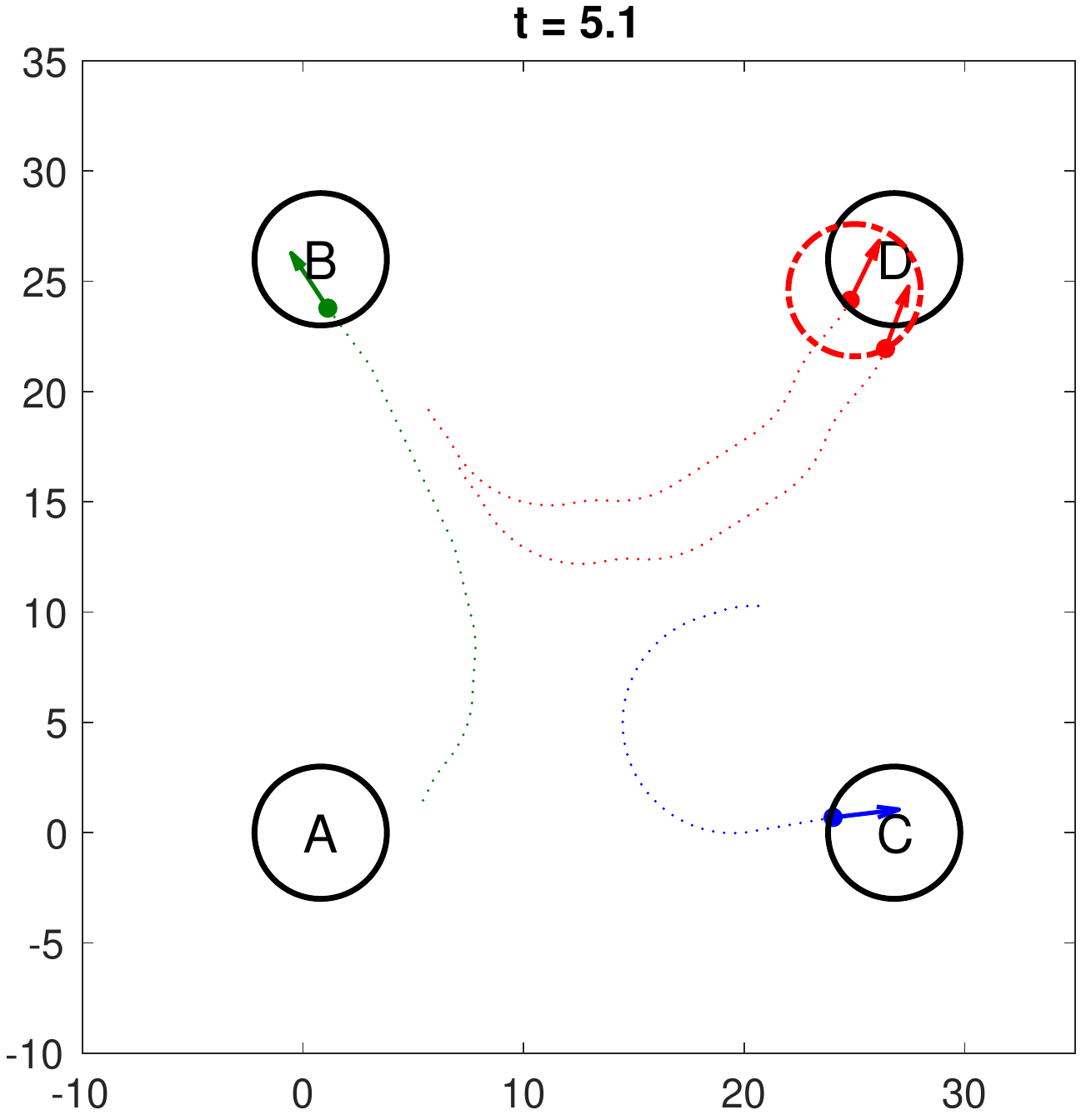}
  \end{subfigure}
  \begin{subfigure}[b]{0.23\textwidth}
    \includegraphics[width=\textwidth]{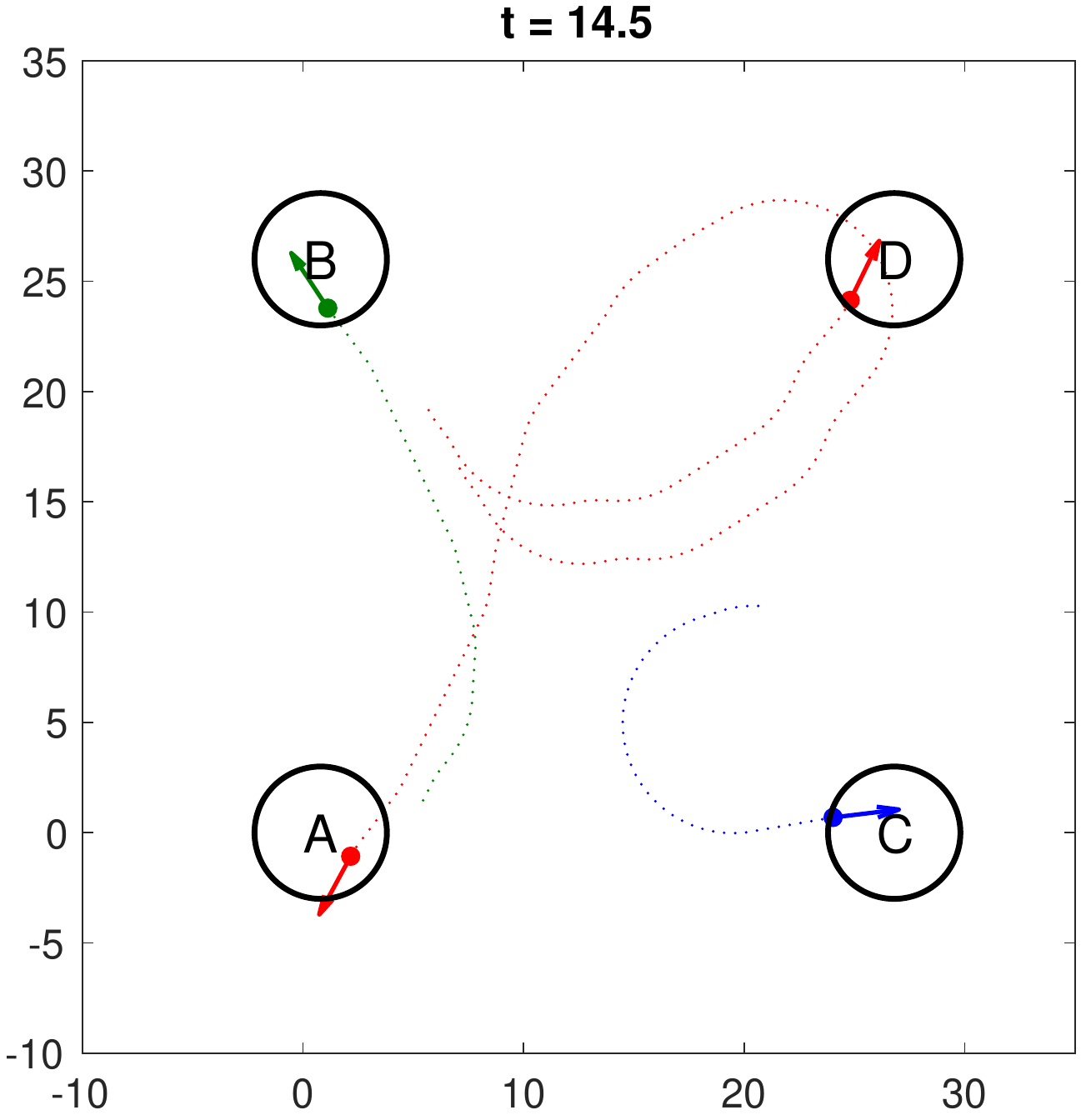}
  \end{subfigure}
  \\
  \begin{subfigure}[b]{0.15\textwidth}
    \includegraphics[width=\textwidth]{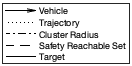}
  \end{subfigure}
  \caption{Four vehicles $\veh_1, \veh_2, \veh_3, \veh_4$ are tasked with visiting their targets. Based on their targets, the team assignment optimization problem described in Section \ref{sec:cluster_assign} assigns $\veh_1$ and $\veh_4$ to cluster $\cluster_1$ (red), $\veh_2$ to cluster $\cluster_2$ (green), and $\veh_3$ (blue) to cluster $\cluster_3$. At $t=1.4$s, the clusters get into potential conflicts with each other and the safety control strategy kicks in to make sure each vehicle remains safe. 
  At $t=14.5$s, we see that each vehicle completes visiting all their targets successfully without any collisions.}
  \label{fig:four_veh}
  \vspace{-0.5cm}
\end{figure}

\begin{figure}[]
\centering
  \begin{subfigure}[b]{0.23\textwidth}
    \includegraphics[width=\textwidth]{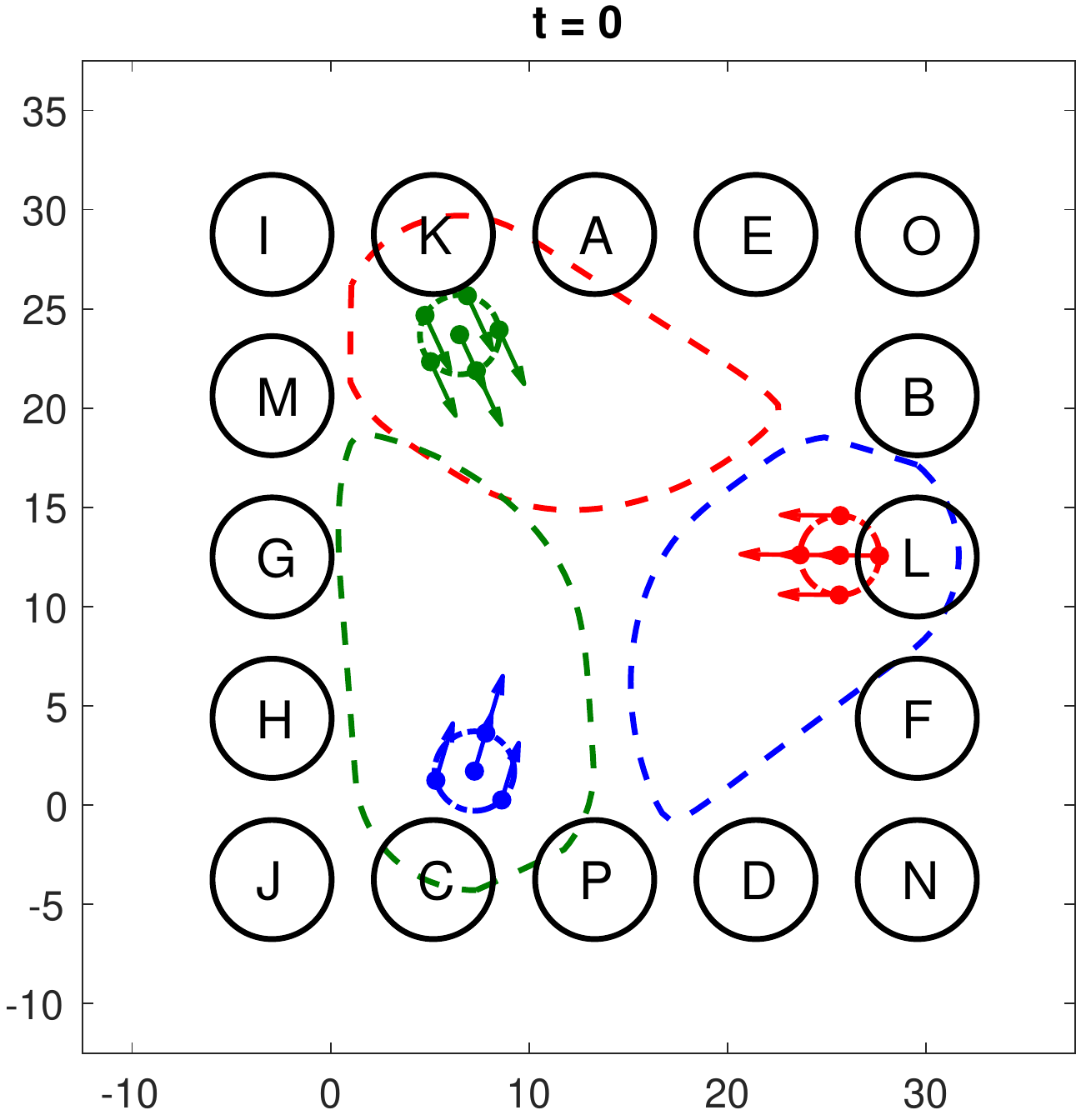}
  \end{subfigure}
  \begin{subfigure}[b]{0.23\textwidth}
    \includegraphics[width=\textwidth]{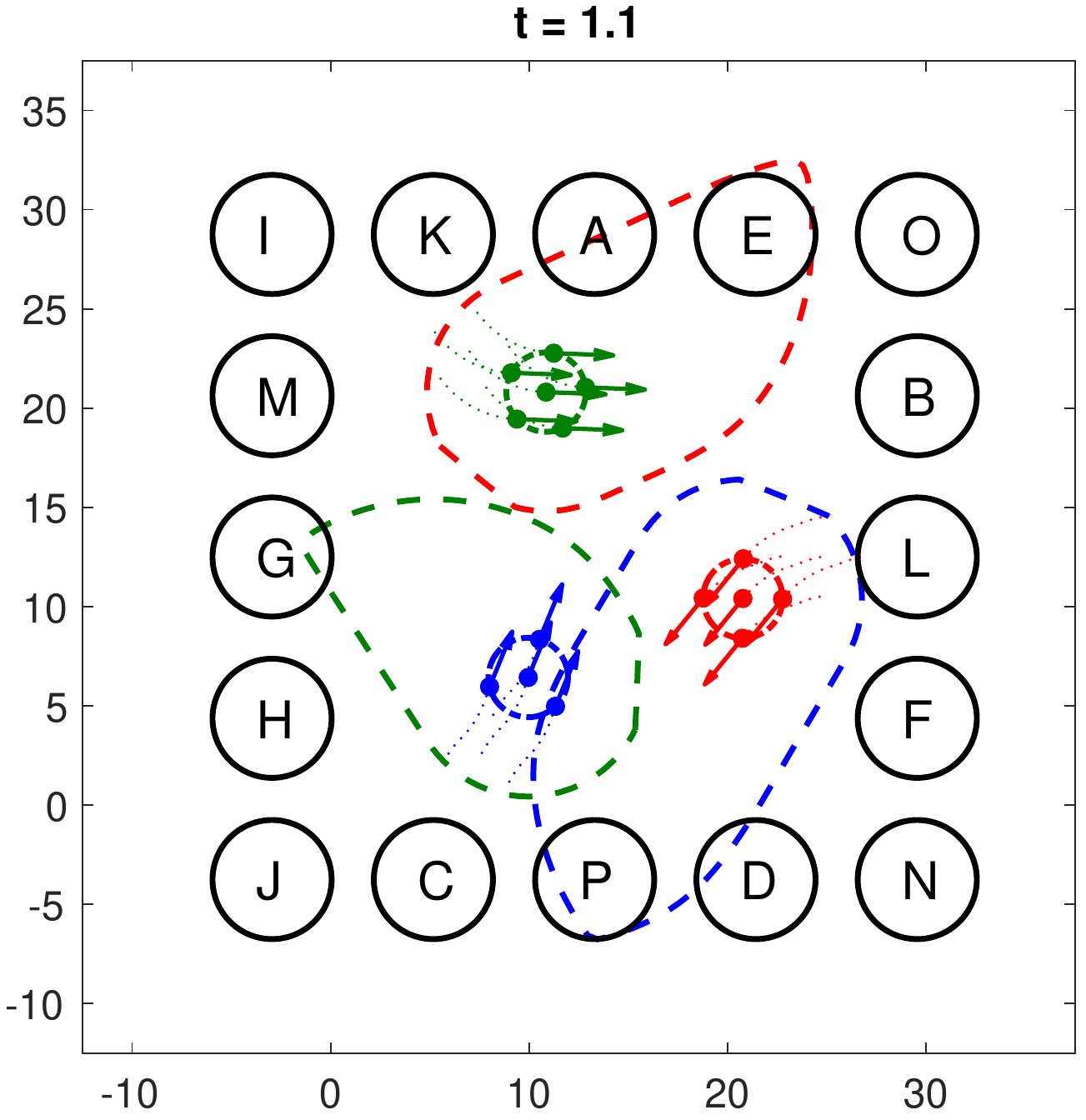}
  \end{subfigure}
  \\
  \begin{subfigure}[b]{0.23\textwidth}
    \includegraphics[width=\textwidth]{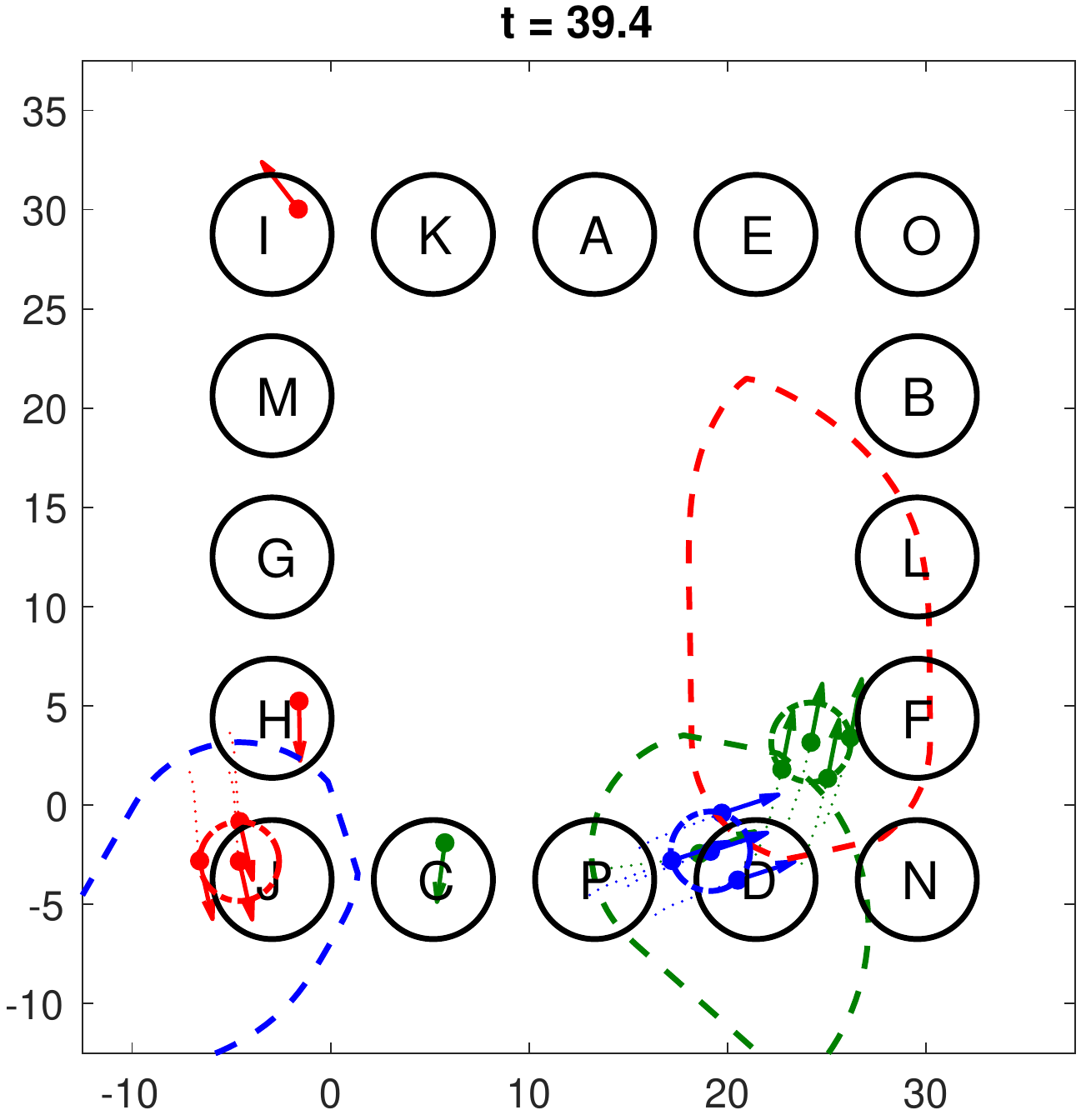}
  \end{subfigure}
  \begin{subfigure}[b]{0.23\textwidth}
    \includegraphics[width=\textwidth]{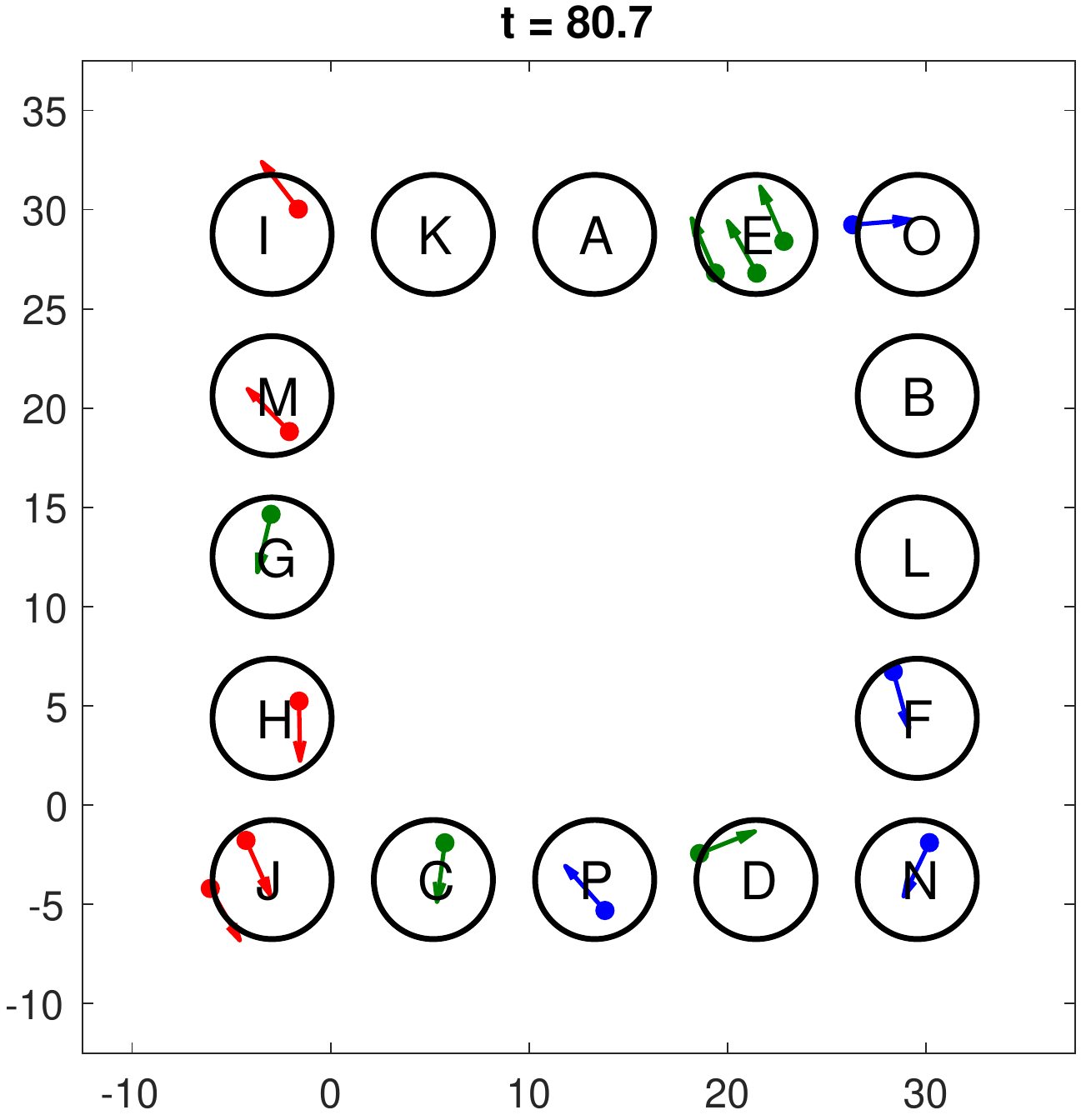}
  \end{subfigure}
  \caption{In this figure, we demonstrate our approach on $15$ vehicles. The vehicles are assigned into three cluster, with cluster $\cluster_1$ (red) having 5 vehicles, cluster $\cluster_2$ (green) having 6 vehicles, and $\cluster_3$ having $4$ vehicles. We can see that the clusters resolve conflicts with each other successfully while they are en route to their targets. At the end, we see that all vehicles safely visited all their targets.} 
  \label{fig:fifteen_veh}
  \vspace{-0.75cm}
\end{figure}

\section{Conclusion and Future Work}
In this paper, we proposed a novel method for \textit{any} number of vehicles to complete their objectives of visiting multiple targets with guaranteed safety for a class of dynamical systems. We demonstrate the effectiveness and scalability of our approach through a $15$-vehicle simulation. 
Future work includes optimizing the order in which the targets are visited if target locations are known \textit{a priori} and developing guaranteed safe control strategies that require less synchronous actions among groups of vehicles for \textit{any} number of vehicles.

 \bibliographystyle{IEEEtran}
 \bibliography{references}

\begin{thebibliography}{10}
\providecommand{\url}[1]{#1}
\csname url@samestyle\endcsname
\providecommand{\newblock}{\relax}
\providecommand{\bibinfo}[2]{#2}
\providecommand{\BIBentrySTDinterwordspacing}{\spaceskip=0pt\relax}
\providecommand{\BIBentryALTinterwordstretchfactor}{4}
\providecommand{\BIBentryALTinterwordspacing}{\spaceskip=\fontdimen2\font plus
\BIBentryALTinterwordstretchfactor\fontdimen3\font minus
  \fontdimen4\font\relax}
\providecommand{\BIBforeignlanguage}[2]{{%
\expandafter\ifx\csname l@#1\endcsname\relax
\typeout{** WARNING: IEEEtran.bst: No hyphenation pattern has been}%
\typeout{** loaded for the language `#1'. Using the pattern for}%
\typeout{** the default language instead.}%
\else
\language=\csname l@#1\endcsname
\fi
#2}}
\providecommand{\BIBdecl}{\relax}
\BIBdecl

\bibitem{Google2020}
\BIBentryALTinterwordspacing
{Google, Inc.} (2020) X project wing. [Online]. Available:
  \url{https://x.company/projects/wing/}
\BIBentrySTDinterwordspacing

\bibitem{Amazon20}
\BIBentryALTinterwordspacing
{Amazon.com, Inc.} (2020) Amazon prime air. [Online]. Available:
  \url{https://www.amazon.com/Amazon-Prime-Air/b?ie=UTF8\&node=8037720011}
\BIBentrySTDinterwordspacing

\bibitem{UPS2020}
\BIBentryALTinterwordspacing
{UPS.} (2020) Ups drone delivery service. [Online]. Available:
  \url{https://www.ups.com/us/en/services/shipping-services/flight-forward-drones.page}
\BIBentrySTDinterwordspacing

\bibitem{Zipline2020}
\BIBentryALTinterwordspacing
{Zipline Inc.} (2020) Zipline medical supply drone delivery. [Online].
  Available: \url{https://flyzipline.com/}
\BIBentrySTDinterwordspacing

\bibitem{Vayu2020}
\BIBentryALTinterwordspacing
{Vayu Inc.} (2020) Vayu medical supply drone delivery. [Online]. Available:
  \url{https://www.vayu.us/}
\BIBentrySTDinterwordspacing

\bibitem{DSLRPros2020}
\BIBentryALTinterwordspacing
{DSLRPros Inc.} (2020) Dslrpros disaster response drone. [Online]. Available:
  \url{https://www.dslrpros.com/disaster-response-drones.html}
\BIBentrySTDinterwordspacing

\bibitem{Humanitarion2020}
\BIBentryALTinterwordspacing
{irevolutions.org}. (2020) Disaster response drone. [Online]. Available:
  \url{https://irevolutions.org/2014/06/25/humanitarians-in-the-sky/}
\BIBentrySTDinterwordspacing

\bibitem{AUVSI16}
\BIBentryALTinterwordspacing
{AUVSI News}. (2016) Uas aid in south carolina tornado investigation. [Online].
  Available: \url{http://www.auvsi.org/blogs/auvsi-news/2016/01/29/tornado}
\BIBentrySTDinterwordspacing

\bibitem{Military2020}
\BIBentryALTinterwordspacing
{Business Insider}. (2020) Drone technology for military. [Online]. Available:
  \url{https://www.businessinsider.com/drone-technology-uses-applications}
\BIBentrySTDinterwordspacing

\bibitem{FAA2020}
\BIBentryALTinterwordspacing
{Federal Administration Regulation}. (2016) Faa drone zone. [Online].
  Available: \url{https://www.faa.gov/uas/}
\BIBentrySTDinterwordspacing

\bibitem{Fiorini98}
P.~Fiorini and Z.~Shillert, ``Motion planning in dynamic environments using
  velocity obstacles,'' \emph{International Journal of Robotics Research},
  vol.~17, pp. 760--772, 1998.

\bibitem{Vandenberg08}
J.~van~den Berg, M.~C. Lin, and D.~Manocha, ``Reciprocal velocity obstacles for
  real-time multi-agent navigation,'' in \emph{IEEE International Conference on
  Robotics and Automation}, May 2008, pp. 1928--1935.

\bibitem{Saber02}
R.~Olfati-Saber and R.~M. Murray, ``Distributed cooperative control of multiple
  vehicle formations using structural potential functions,'' in \emph{IFAC
  World Congress}, 2002.

\bibitem{Chuang07}
Y.-L. Chuang, Y.~Huang, M.~R. D'Orsogna, and A.~L. Bertozzi, ``Multi-vehicle
  flocking: Scalability of cooperative control algorithms using pairwise
  potentials,'' in \emph{IEEE International Conference onRobotics and
  Automation}, April 2007, pp. 2292--2299.

\bibitem{Mitchell05}
I.~Mitchell, A.~Bayen, and C.~Tomlin, ``A time-dependent {Hamilton-Jacobi}
  formulation of reachable sets for continuous dynamic games,'' \emph{IEEE
  Transactions on Automatic Control}, vol.~50, no.~7, pp. 947--957, 2005.

\bibitem{Fisac15}
J.~F. Fisac, M.~Chen, C.~J. Tomlin, and S.~S. Shankar, ``Reach-avoid problems
  with time-varying dynamics, targets and constraints,'' in \emph{18th
  International Conference on Hybrid Systems: Computation and Controls}, 2015.

\bibitem{Tanimoto78}
S.~Tanimoto, ``On a class of three-player differential games,'' \emph{Journal
  of Optimization Theory and Applications}, vol.~25, no.~3, p. 469?473, 1978.

\bibitem{Su14}
M.~Su, Y.~ji~Wang, and L.~Liu, ``Bounded guidance law based on differential
  game for three-player conflict,'' in \emph{IEEE Conference on Modeling,
  Identification, and Control}, 2014.

\bibitem{Fisac15b}
J.~F. Fisac and S.~S. Sastry, ``The pursuit-evasion-defense differential game
  in dynamic constrained environments,'' in \emph{IEEE Conference on Decision
  and Control}, 2015.

\bibitem{Shih16}
M.~Chen, J.~C. Shih, and C.~J. Tomlin, ``Multi-vehicle collision avoidance via
  hamilton-jacobi reachability and mixed integer programming,'' in \emph{55th
  {IEEE} Conference on Decision and Control, {CDC} 2016, Las Vegas, NV, USA,
  December 12-14, 2016}, 2016, pp. 1695--1700.

\bibitem{Chen17}
A.~{Dhinakaran}, M.~{Chen}, G.~{Chou}, J.~C. {Shih}, and C.~J. {Tomlin}, ``A
  hybrid framework for multi-vehicle collision avoidance,'' in \emph{2017 IEEE
  56th Annual Conference on Decision and Control (CDC)}, 2017, pp. 2979--2984.

\bibitem{Chen15}
M.~Chen, J.~Fisac, C.~J. Tomlin, and S.~Sastry, ``Safe sequential path planning
  of multi-vehicle systems via double-obstacle hamilton-jacobi-isaacs
  variational inequality,'' in \emph{European Control Conference}, 2015.

\bibitem{Chen15b}
M.~Chen, Q.~Hu, C.~Mackin, J.~Fisac, and C.~J. Tomlin, ``Safe platooning of
  unmanned aerial vehicles via reachability,'' in \emph{IEEE Conference on
  Decision and Control}, 2015.

\bibitem{Royo19}
V.~{Rubies-Royo}, D.~{Fridovich-Keil}, S.~{Herbert}, and C.~J. {Tomlin}, ``A
  classification-based approach for approximate reachability,'' in \emph{2019
  International Conference on Robotics and Automation (ICRA)}, 2019, pp.
  7697--7704.

\bibitem{Gillula11}
J.~H. Gillula, G.~M. Hoffmann, H.~Huang, M.~P. Vitus, and C.~J. Tomlin,
  ``Applications of hybrid reachability analysis to robotic aerial vehicles,''
  in \emph{The International Journal of Robotics Research}, vol.~30, no.~3,
  2011, pp. 335--354.

\bibitem{Bouffard12}
P.~Bouffard, ``On-board model predictive control of a quadrotor helicopter:
  Design, implementation, and experiments,'' in \emph{Master’s
  thesis,University of California, Berkeley}, 2012.

\end{thebibliography}
  
 \clearpage

\end{document}